\documentclass[10pt,twocolumn,letterpaper]{article}

\usepackage{cvpr}
\usepackage{times}
\usepackage{epsfig}
\usepackage{graphicx}
\usepackage{amsmath}
\usepackage{amssymb}

\usepackage{amsthm}
\usepackage{multirow, paralist}

\usepackage{array}
\makeatletter
\newcommand{\thickhline}{%
    \noalign {\ifnum 0=`}\fi \hrule height 1pt
    \futurelet \reserved@a \@xhline
}

\newcolumntype{"}{@{\hskip\tabcolsep\vrule width 1pt\hskip\tabcolsep}}
\makeatother

\def \OO {\mathcal{O}}

\def \q {\mathbf{q}}
\def \oo {\mathbf{o}}

\def \q {\mathbf{q}}

\def \LL {\mathcal{L}}
\newtheorem{definition}{Definition}

\newtheorem{thm}{Theorem}
\newtheorem{prop}{Proposition}
\newtheorem{cor}{Corollary}

\usepackage[breaklinks=true,bookmarks=false]{hyperref}

\cvprfinalcopy 

\def\cvprPaperID{9705} 

\ifcvprfinal\fi
\begin{document}

\title{DR Loss: Improving Object Detection by Distributional Ranking}

\author{Qi Qian$^1$\quad Lei Chen$^2$\quad Hao Li$^2$\quad Rong Jin$^1$\\
Alibaba Group\\
$^1$Bellevue, WA, 98004, USA\\
$^2$Hangzhou, China\\
{\tt\small \{qi.qian, fanjiang.cl, lihao.lh, jinrong.jr\}@alibaba-inc.com}
}

\maketitle
\thispagestyle{empty}

\begin{abstract}
Most of object detection algorithms can be categorized into two classes: two-stage detectors and one-stage detectors. Recently, many efforts have been devoted to one-stage detectors for the simple yet effective architecture. Different from two-stage detectors, one-stage detectors aim to identify foreground objects from all candidates in a single stage. This architecture is efficient but can suffer from the imbalance issue with respect to two aspects: the inter-class imbalance between the number of candidates from foreground and background classes and the intra-class imbalance in the hardness of background candidates, where only a few candidates are hard to be identified. In this work, we propose a novel distributional ranking (DR) loss to handle the challenge. For each image, we convert the classification problem to a ranking problem, which considers pairs of candidates within the image, to address the inter-class imbalance problem. Then, we push the distributions of confidence scores for foreground and background towards the decision boundary. After that, we optimize the rank of the expectations of derived distributions in lieu of original pairs. Our method not only mitigates the intra-class imbalance issue in background candidates but also improves the efficiency for the ranking algorithm. By merely replacing the focal loss in RetinaNet with the developed DR loss and applying ResNet-101 as the backbone, mAP of the single-scale test on COCO can be improved from $39.1\%$ to $41.7\%$ without bells and whistles, which demonstrates the effectiveness of the proposed loss function. Code is available at \url{https://github.com/idstcv/DR_loss}.
\end{abstract}
\section{Introduction}
The performance of object detection has been improved dramatically with the development of deep neural networks in the past few years. 
\begin{figure}[!ht]
\centering
\includegraphics[width=3.2in]{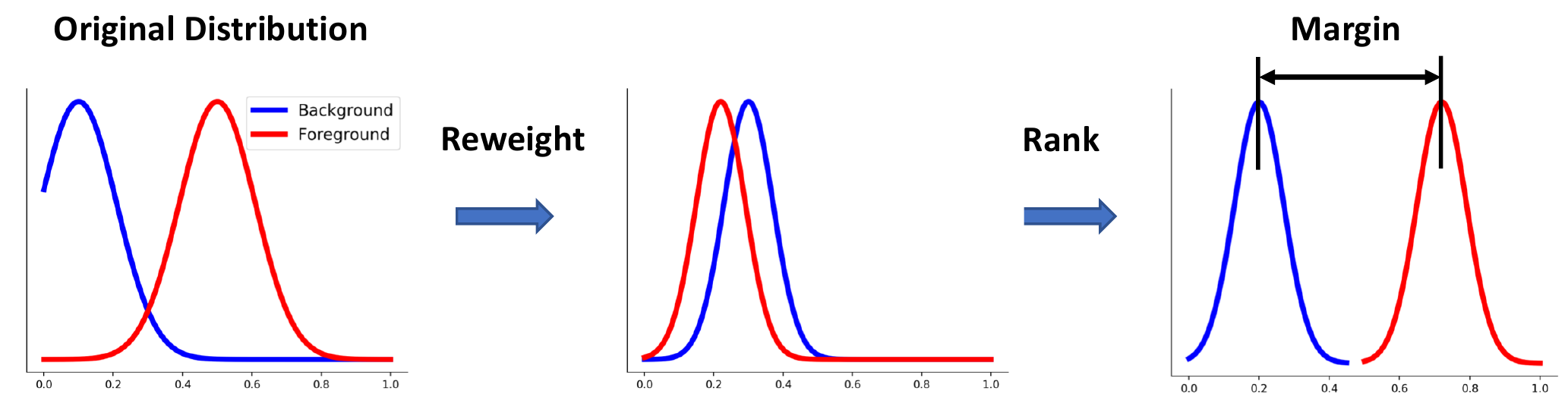}
\caption{Illustration of the proposed distributional ranking loss. First, we push the distributions of confidence scores towards the decision boundary by re-weighting examples. Then, we try to rank the expectation of the derived distribution of foreground above that of background by a large margin.\label{fig:illu}}
\end{figure}
Most of detection algorithms fall into two categories: two-stage detectors~\cite{DaiQXLZHW17,HeGDG17,HeZRS16,LinDGHHB17} and one-stage detectors~\cite{ChenLLSWD0HZ19,LawD18,LinGGHD17,LiuAESRFB16,RedmonF17,tian2019fcos, ZhuHS19}. For the two-stage schema, the procedure of the algorithms can be divided into two parts. 
In the first stage, a region proposal method filters most of background candidate bounding boxes and keeps only a small set of candidates. In the following stage, these candidates are classified as specific foreground classes or background and the bounding boxes will be further refined by minimizing a regression loss. Two-stage detectors demonstrate the superior performance on real-world data sets while the efficiency can be an issue in practice, especially for the devices with limited computing resources, e.g., smartphones, cameras, etc.

Therefore, one-stage detectors are developed for the efficient detection. Different from two-stage detectors, one-stage methods consist of a single phase and have to identify foreground objects from all candidates directly. The procedure of one-stage detectors is straightforward and efficient. However, one-stage detectors can suffer from the imbalance problem that can reside in the following two aspects. First, the numbers of candidates between classes are imbalanced. Without a region proposal phase, the number of background candidates can easily overwhelm that of foreground ones. Second, the hardness of identification for background candidates is imbalanced. Most of them can be easily identified from foreground objects while only a few of them are hard to be classified.

To mitigate the imbalance problem, SSD~\cite{LiuAESRFB16} adopts hard negative mining in training, which is a popular strategy~\cite{ShrivastavaGG16,ViolaJ01} to keep a small set of background candidates with the highest loss. By eliminating simple background candidates, the strategy balances the number of candidates between classes and the hardness of background simultaneously. However, certain information from background can be lost, and thus the detection performance can degrade as illustrated in \cite{LinGGHD17}. RetinaNet~\cite{LinGGHD17} proposes to keep all background candidates but assign different weights for the loss functions of candidates. The weighted cross entropy loss is referred as focal loss. It makes the algorithm focus on the hard candidates while reserving the information from all candidates. This strategy improves the performance of one-stage detectors significantly. Despite the success of focal loss, it re-weights classification losses in a heuristic way and can be insufficient to address the imbalance problem. Moreover, focal loss is designed for a single \textit{candidate} and is image-independent while object detection aims to identify objects in a single \textit{image}. Focal loss lacks the exploration for each image as a whole and the inconsistency can make the performance suboptimal.


In this work, we propose an image-dependent ranking loss to handle the imbalance challenge. First, to mitigate the effect of the inter-class imbalance problem, we convert the classification problem to a ranking problem, which considers ranks of pairs. Since each pair consists of a foreground candidate and a background candidate, it is well balanced. Moreover, considering the intra-class imbalance in hardness of background candidates, we design the distributional ranking (DR) loss to rank the distribution of confidence scores for foreground above that for background candidates. As illustrated in Fig.~\ref{fig:illu}, we first push the original distributions towards the decision boundary with appropriate constraints. After obtaining the drifted distributions, we can rank the expectations of distributions in lieu of original examples to identify foreground from background, which improves the efficiency by reducing the number of pairs from $\OO(n^2)$ to $\OO(1)$ in ranking, where $n$ is the number of candidates in an image. Compared with focal loss, DR loss is image-dependent and can explore the information within each image sufficiently.

We conduct experiments on COCO~\cite{LinMBHPRDZ14} to demonstrate the proposed DR loss. Since the focal loss is designed as the classification loss in RetinaNet, we adopt RetinaNet as the base detector for a fair comparison. Specifically, we merely replace the focal loss with the DR loss while keeping other components unchanged. With ResNet-101~\cite{HeZRS16} as the backbone, minimizing our loss function can boost the mAP of RetinaNet from $39.1\%$ to $41.7\%$, which confirms the effectiveness of the proposed loss.

The rest of this paper is organized as follows. Section \ref{sec:related} reviews the related work in object detection. Section \ref{sec:method} describes the details of the proposed DR loss. Section \ref{sec:exp} compares our method to others on COCO detection task. Finally, Section \ref{sec:conclusion} concludes this work.

\section{Related Work}
\label{sec:related}
Detection is a fundamental task in computer vision. In conventional methods, hand crafted features, e.g., HOG~\cite{DalalT05} and SIFT~\cite{Lowe04}, are used for detection either with a sliding-window strategy which holds a dense set of candidates, e.g., DPM~\cite{FelzenszwalbMR08} or with a region proposal method which keeps a sparse set of candidates, e.g., Selective Search~\cite{UijlingsSGS13}. Recently, deep neural networks have shown the dominating performance in classification tasks~\cite{KrizhevskySH12}, and the features obtained from neural networks are leveraged for detection tasks. 

R-CNN~\cite{GirshickDDM14} equips the region proposal stage and works as a two-stage algorithm. It first obtains a sparse set of regions by selective search. In the next stage, a deep convolutional neural network is applied to extract features for each region. Finally, regions are classified with a conventional classifier, e.g., SVM. R-CNN improves the performance of detection by a large margin but the procedure is too slow for real-world applications. Hence, many variants are developed to accelerate it~\cite{Girshick15,RenHGS15}. To further improve the accuracy, Mask-RCNN~\cite{HeGDG17} adds a branch for object mask prediction to boost the performance with the additional information from multi-task learning. Besides the two-stage structure, Cascade R-CNN~\cite{CaiV18} develops a multi-stage strategy to promote the quality of detectors after the region proposal stage in a cascade fashion.

One-stage detectors are developed for efficiency~\cite{ChenLLSWD0HZ19,LiuAESRFB16,RedmonDGF16,SermanetEZMFL13,ZhuHS19}. Since there is no region proposal phase to sample background candidates, one-stage detectors can suffer from the imbalance issue from both the inter-class imbalance between foreground and background candidates and intra-class imbalance in the background candidates. To address the challenge, SSD~\cite{LiuAESRFB16} adopts hard negative mining, which only keeps a small set of hard background candidates for training. Recently, focal loss~\cite{LinGGHD17} is proposed to handle the problem in RetinaNet. Unlike SSD, it keeps all background candidates but re-weights them such that the hard examples are assigned with a large weight. Focal loss improves the performance of one-stage detection explicitly, but the imbalance problem in detection is still not sufficiently explored. Besides those anchor-based algorithms, anchor-free one-stage detectors~\cite{tian2019fcos,ZhuHS19} have been developed, where focal loss is also applied for classification. The work closest to ours is the AP-loss in \cite{ChenLLSWD0HZ19}, where a ranking loss is designed to optimize the average precision. However, the loss focuses on the original pairs and is non-differentiable. A specific algorithm has to be developed to minimize the AP-loss. In this work, we develop the DR loss that ranks the expectations of distributions in lieu of original pairs. DR loss is differentiable and can be optimized with stochastic gradient descent (SGD) in the standard training pipeline. Therefore, our loss can work in a plug and play manner, which is important for real-world applications.

\section{DR Loss}
\label{sec:method}

Given a set of candidate bounding boxes from an image, a detector has to identify the foreground objects from background ones with a classification model. Let $\theta$ denote a classifier and it can be learned by optimizing the problem
\begin{eqnarray}\label{eq:origin}
\min\limits_{\theta} \sum_i^N\sum_{j,k} \ell(p_{i,j,k})
\end{eqnarray}
where $N$ is the number of total images. In this work, we employ sigmoid function to predict the probability for each candidate. $p_{i,j,k}$ is the prediction from $\theta$ and indicates the estimated probability that the $j$-th candidate in the $i$-th image is from the $k$-th class. $\ell(\cdot)$ is the loss function. In most of detectors, the classifier is learned by minimizing the cross entropy loss or its variants. 

The objective in Eqn.~\ref{eq:origin} is prevalent but can suffer from the inter-class imbalance problem. The problem can be demonstrated by rewriting the original problem as
\begin{eqnarray}\label{eq:imba}
\min_{\theta} \sum_i^N(\sum_{j_+}^{n_+} \ell(p_{i,j_+})+\sum_{j_-}^{n_-}\ell(p_{i,j_-}))
\end{eqnarray}
where $j_+$ and $j_-$ denote the positive (i.e., foreground) and negative (i.e., background) examples (e.g., anchors), respectively. $n_+$ and $n_-$ are the corresponding number of examples. When $n_-\gg n_+$, the accumulated loss from the latter term will dominate. This issue is from the fact that the losses for positive and negative examples are separated and the contribution from positive examples will be overwhelmed by negative ones. One heuristic to handle the problem is emphasizing positive examples, which can change the weights for the corresponding losses. In this work, we aim to address the problem in a fundamental way. For brevity, we will omit the index of image (i.e., $i$) from the next subsection.

\subsection{Ranking}
\label{sec:orank}

To mitigate the challenge from the imbalance between classes, we consider to optimize the rank between positive and negative examples. Given a pair of positive and negative examples, an ideal ranking model can rank the positive example above the negative one with a large margin
\begin{eqnarray*}
\forall j_+, j_-\quad p_{j_+}-p_{j_-}\geq \gamma
\end{eqnarray*}
where $\gamma$ is a non-negative constant. Compared with the objective in Eqn.~\ref{eq:origin}, the ranking model optimizes the relationship between individual positive and negative examples, which is well balanced. 

The objective of ranking for a single image can be written as
\begin{eqnarray}\label{eq:rank}
\min_\theta \sum_{j_+}^{n_+}\sum_{j_-}^{n_-} \ell(p_{j_-} - p_{j_+}+\gamma)
\end{eqnarray}
where the hinge loss is applied as the loss function
\[\ell_{\mathrm{hinge}}(z)=[z]_+=\left\{\begin{array}{ll}z&z> 0\\ 0&o.w.\end{array}\right.\]
The objective can be interpreted by the equivalent form
\begin{align}\label{eq:exp}
&\frac{1}{n_+n_-}\sum_{j_+}^{n_+}\sum_{j_-}^{n_-} \ell(p_{j_-} - p_{j_+}+\gamma) \nonumber\\
&= E_{j_+,j_-}[\ell(p_{j_-} - p_{j_+}+\gamma)]
\end{align}
It demonstrates that the objective measures the expectation of the ranking loss on a randomly sampled pair. 

The ranking loss addresses the inter-class imbalance issue by comparing the rank of each positive example to negative examples. However, it ignores a phenomenon in object detection, where the hardness of negative examples is also imbalanced. Besides, the ranking loss introduces a new challenge, that is, the vast number of pairs. We tackle them in the following subsection. 

\subsection{Distributional Ranking}

As indicated in Eqn.~\ref{eq:exp}, the ranking loss in Eqn.~\ref{eq:rank} punishes a mis-ranking for a uniformly sampled pair. In detection, most of negative examples can be easily ranked well, that is, a randomly sampled pair will not incur the ranking loss with high probability. Therefore, we consider to optimize the ranking boundary to avoid the trivial solution
\begin{eqnarray}\label{eq:bound}
\min_\theta  \ell(\max_{j_-} \{p_{j_-}\} - \min_{j_+}\{p_{j_+}\}+\gamma)
\end{eqnarray}
If we can rank the positive example with the lowest score above the negative one with the highest confidence, the whole set of examples in an image are perfectly ranked. The pair in Eqn.~\ref{eq:bound} is referred as the worst-case scenario, which will incur the largest loss among all pairs. Compared with the conventional ranking loss, optimizing the loss from the worst-case scenario is much more efficient, which reduces the number of pairs from $n_+n_-$ to $1$. Moreover, it clearly eliminates the inter-class imbalance issue since only a single pair of positive and negative examples is required for each image. However, this formulation is very sensitive to the selected pair, which can result in the degraded detection model.

To improve the robustness, we first introduce the distribution of confidence scores for the positive and negative examples and obtain the expectation as
\begin{eqnarray*}
\quad P_{+}= \sum_{j_+}^{n_+}q_{j_+}p_{j_+};\quad P_{-} =  \sum_{j-}^{n_-}q_{j_-}p_{j_-}
\end{eqnarray*}
where $\q_{+}\in\Delta$ and $\q_{-}\in\Delta$ denote the distributions over positive and negative examples, respectively. $P_{+}$ and $P_{-}$ represent the expected scores under the corresponding distribution. $\Delta$ is the simplex as $\Delta=\{q: \sum_j q_j=1,\forall j, q_j\geq 0 \}$. When $\q_{+}$ and $\q_{-}$ are the uniform distribution, $P_{+}$ and $P_{-}$ demonstrate the expectation from the original distribution. 

With these definitions, the distribution corresponding to the worst-case scenario can be derived as
\[ P_{+}= \min_{\q_{+}\in\Delta}\sum_{j_+}^{n_+}q_{j_+}p_{j_+};\ \ P_{-} =  \max_{\q_{-}\in\Delta}\sum_{j-}^{n_-}q_{j_-}p_{j_-}\]
We can rewrite the problem in Eqn.~\ref{eq:bound} in the equivalent form
\[\min_\theta  \ell(P_{-}-P_{+}+\gamma)\]
which can be considered as ranking the distributions between positive and negative examples in the worst-case scenario. 

By investigating the new formulation, it is obvious that optimizing the worst-case scenario is not robust due to the fact that the domain of the generated distribution is unconstrained. Consequently, it will concentrate on a single example while ignoring the influence of the original distribution that contains massive information. Hence, we improve the robustness of the loss by regularizing the freedom of the derived distribution
\begin{eqnarray*}
P_{-} &=& \max_{\q_{-}\in\Delta,\Omega(\q_{-}||\oo_-)\leq \epsilon_-} \sum_{j-}^{n_-}q_{j_-}p_{j_-} \\
-P_{+}&=& \max_{\q_{+}\in\Delta,\Omega(\q_{+}||\oo_+)\leq \epsilon_+} \sum_{j_+}^{n_+}q_{j_+}(-p_{j_+})
\end{eqnarray*}
where $\oo_+, \oo_-$ denote the original distributions for positive and negative examples, respectively. $\Omega(\cdot)$ is a regularizer for the diversity of the distribution to prevent the distribution from the trivial one-hot solution. It measures the similarity between the generated distribution and the original distribution, and some popular similarity function can be applied, e.g., $L_p$ distance, R\'{e}nyi entropy, Shannon entropy, etc. $\epsilon_-$ and $\epsilon_+$ are constants to control the freedom of derived distributions.

To obtain the constrained distribution, we consider the subproblem
\begin{eqnarray*}
&&\max_{\q_{-}\in\Delta} \sum_{j-}q_{j_-}p_{j_-}\\
s.t.&&\Omega(\q_{-}||\oo_-)\leq \epsilon_-
\end{eqnarray*}
According to the dual theory~\cite{boyd2004convex}, given $\epsilon_-$, we can find the parameter $\lambda_-$ to obtain the optimal $\q_{-}$ by solving the problem
\begin{eqnarray*}
\max_{\q_{-}\in\Delta} \sum_{j-}q_{j_-}p_{j_-} -\lambda_-\Omega(\q_{-}||\oo_-)
\end{eqnarray*}
We observe that the former term is linear in $\q_{-}$. Hence, if $\Omega(\cdot)$ is convex in $\q_{-}$, the problem can be solved efficiently by first order algorithms~\cite{boyd2004convex}. In this work, we adopt KL-divergence as the regularizer and have the closed-form solution as follows
\begin{prop}\label{prop:1}
For the problem
\begin{eqnarray*}
\max_{\q_{-}\in\Delta} \sum_{j-}q_{j_-}p_{j_-} -\lambda_-\mathrm{KL}(\q_{_-}||\oo_-)
\end{eqnarray*}
we have the closed-form solution as
\[q_{j_-} = \frac{1}{Z_-}o_{j_-}\exp(\frac{p_{j_-}}{\lambda_-});\quad Z_- = \sum_{j_-}o_{j_-}\exp(\frac{p_{j_-}}{\lambda_-})\]
\end{prop}
\begin{proof}
It can be proved directly from K.K.T. condition~\cite{boyd2004convex}.
\end{proof}
For the distribution over positive examples, we have the similar result as
\begin{prop}
For the problem
\begin{eqnarray*}
\max_{\q_{+}\in\Delta} \sum_{j+}q_{j_+}(-p_{j_+}) -\lambda_+\mathrm{KL}(\q_{+}||\oo_+)
\end{eqnarray*}
we have the closed-form solution as
\[q_{j_+} = \frac{1}{Z_+}o_{j_+}\exp(\frac{-p_{j_+}}{\lambda_+});\quad Z_+ = \sum_{j_+}o_{j_+}\exp(\frac{-p_{j_+}}{\lambda_+})\]
\end{prop}
\paragraph{Remark 1}
These Propositions show that the harder the example, the larger the weight of the example. Besides, the weight is image-dependent and will be affected by other examples in the same image.

The original distributions (i.e., $\oo_-$ and $\oo_+$) can also influence the derived distributions by weighting each candidate. Therefore, the prior knowledge about the problem can be encoded into the original distributions, which makes generating new distributions more flexible. Here we take $\oo_-$ as an example to illustrate different distributions.

\begin{compactitem}
\item Uniform distribution: It means that $\forall j,\ o_{j-} = 1/n_-$. With the constant value, the closed-form solution can be simplified as 
$q_{j_-} = \frac{1}{Z_-}\exp(\frac{p_{j_-}}{\lambda_-});\quad Z_- = \sum_{j_-}\exp(\frac{p_{j_-}}{\lambda_-})$
\item Hard negative mining: In this scenario, we assume $\forall j, \ o_{j-}\in\{0,1/\hat{n}_{-}\}$, where $\hat{n}_-$ denotes the number of non-zero elements in $\oo_-$. According to Proposition~\ref{prop:1}, only candidates selected by $\oo_-$ will be accumulated to derive the new distribution. Therefore, our formulation can incorporate with hard negative mining by setting the weights in $\oo_-$ appropriately.
\end{compactitem}

To keep the loss function simple, we adopt the uniform distribution in this work. Fig.~\ref{fig:ndist} illustrates the changing of the distribution with the proposed strategy. The derived distribution approaches the distribution corresponding to the worst-case scenario when decreasing $\lambda$. Note that the original distributions in Fig.~\ref{fig:ndist} (a) and Fig.~\ref{fig:ndist} (b) have the same mean but different variance. For the distribution with the small variance as in Fig.~\ref{fig:ndist} (a), we can observe that the regularizer $\lambda$ should be small to change the distribution effectively. When the distribution has the large variance, Fig.~\ref{fig:ndist} (b) shows that a large $\lambda$ is sufficient to change the shape of the distribution dramatically. Considering that the distributions of positive and negative examples have different variances,  Fig.~\ref{fig:ndist} implies that different weights for the regularizers should be assigned.

\begin{figure}[!ht]
\centering
\begin{minipage}{1.6in}
\centering
\includegraphics[width=1.5in]{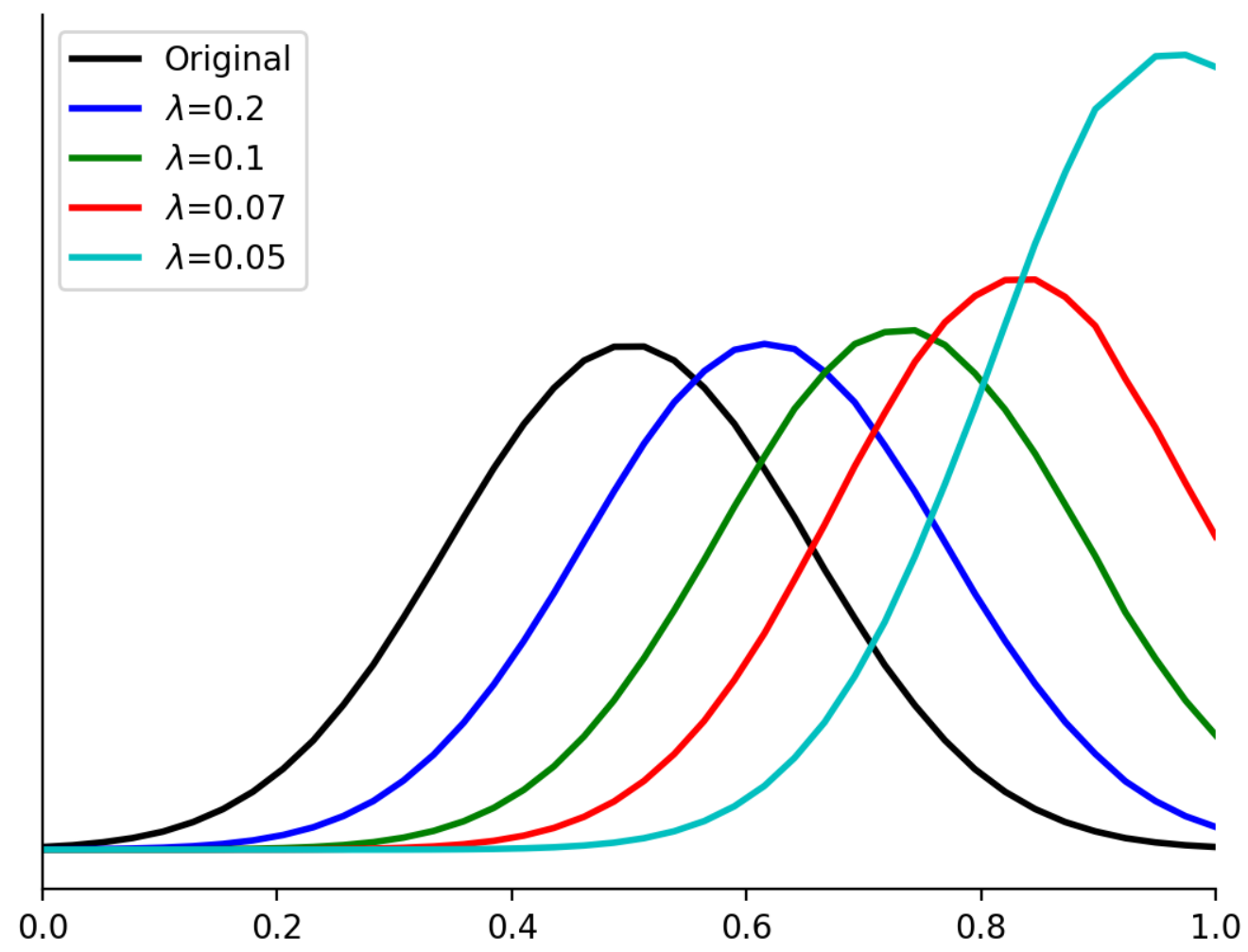}
\mbox{\footnotesize (a) Small Variance}
\end{minipage}
\begin{minipage}{1.6in}
\centering
\includegraphics[width=1.5in]{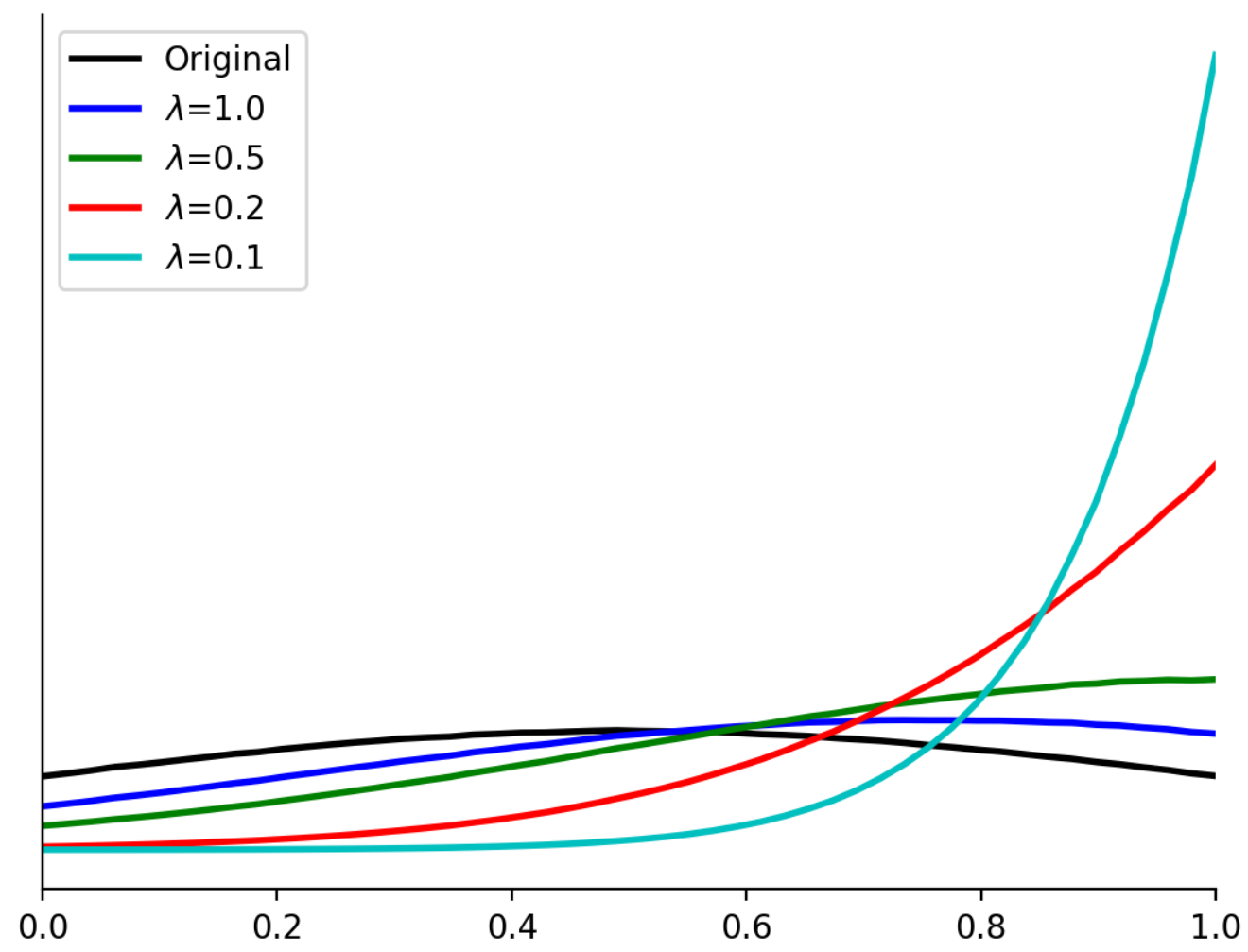}
\mbox{\footnotesize (b) Large Variance}
\end{minipage}
\caption{Illustration of the drifting in the distribution. We randomly sample $1e7$ points from a Gaussian distribution with different variances to mimic scores of anchors. We change the weights of examples according to the proposed strategy as in Proposition~\ref{prop:1} and then plot the curves of different probability density functions (PDF) when varying $\lambda$.\label{fig:ndist}}
\end{figure}

With the closed-form solutions of distributions, the expectation of distributions can be computed as
\begin{align}\label{eq:mexample}
&\hat{P}_{-} = \sum_{j-}^{n_-}q_{j_-}p_{j_-}=\sum_{j-}^{n_-}\frac{1}{Z_-}\exp(\frac{p_{j_-}}{\lambda_-})p_{j_-} \\
&\hat{P}_{+}=\sum_{j-}^{n_-}q_{j_+}p_{j_+}= \sum_{j_+}^{n_+}\frac{1}{Z_+}\exp(\frac{-p_{j_+}}{\lambda_+})p_{j_+}\nonumber
\end{align}

Finally, smoothness is crucial for the convergence of non-convex optimization~\cite{GhadimiL13a}. So we apply the smooth approximation instead of the original hinge loss as the loss function for pairs. The popular substitutes to the hinge loss include quadratic loss and logistic loss
\begin{eqnarray}\label{eq:qloss}
\ell_{\mathrm{quad}}(z) = \left\{\begin{array}{cc}z&z\geq \rho\\ \frac{(z+\rho)^2}{4\rho}&-\rho<z<\rho\\0&z\leq -\rho\end{array}\right.
\end{eqnarray}
\begin{eqnarray}\label{eq:Lloss}
\ell_{\mathrm{logistic}}(z) = \frac{1}{L}\log(1+\exp(Lz))
\end{eqnarray}
where $\rho$ and $L$ control the approximation error of the function. The larger the $L$ is , the closer to the hinge loss the approximation is. $\rho$ works in an opposite direction. Fig.~\ref{fig:lossC} compares the hinge loss to its smooth variants. Explicitly, these functions share the similar shape and we adopt the logistic loss in this work.

\begin{figure}[!ht]
\centering
\begin{minipage}{1.6in}
\centering
\includegraphics[width=1.5in]{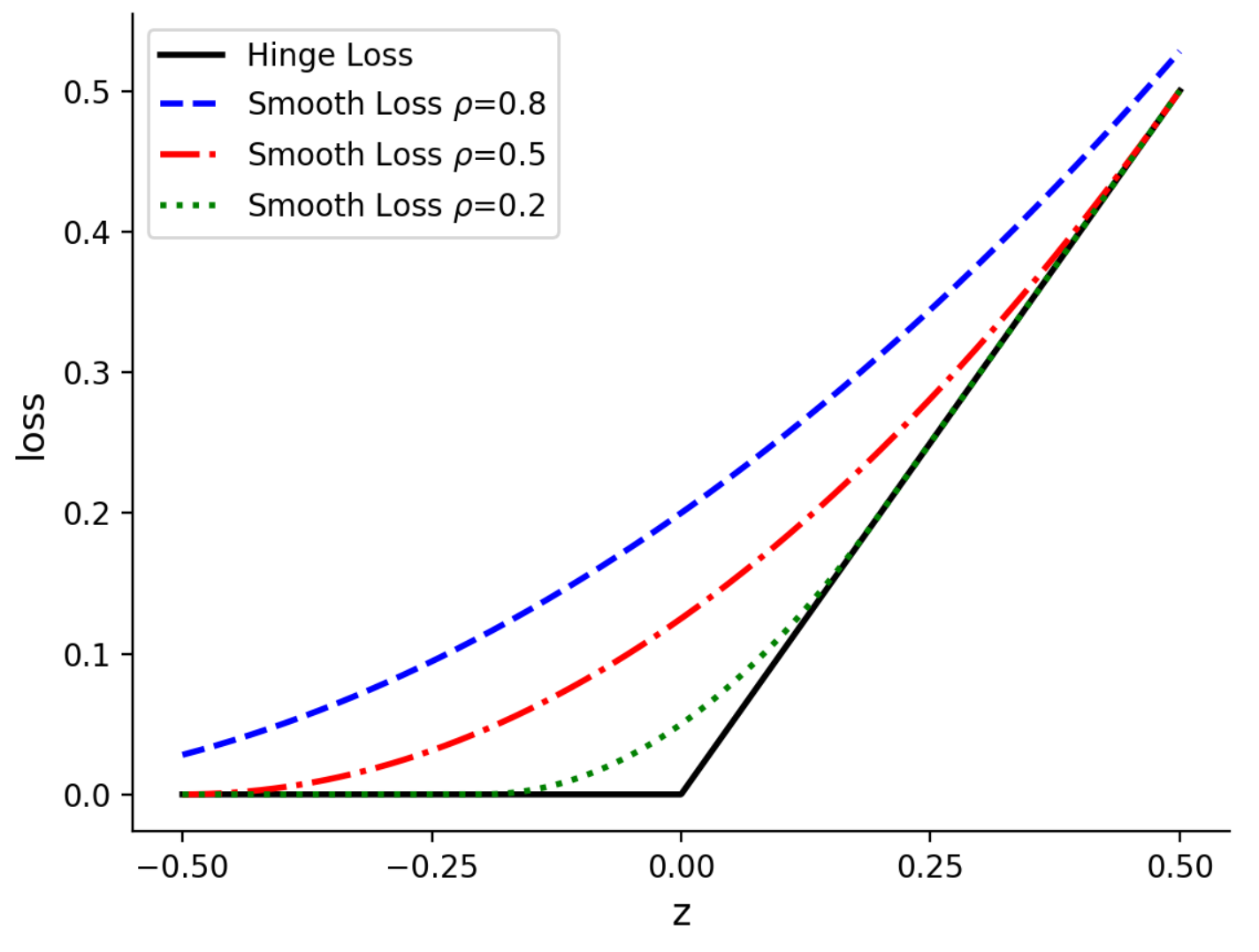}
\mbox{\footnotesize (a) Quadratic Loss}
\end{minipage}
\begin{minipage}{1.6in}
\centering
\includegraphics[width=1.5in]{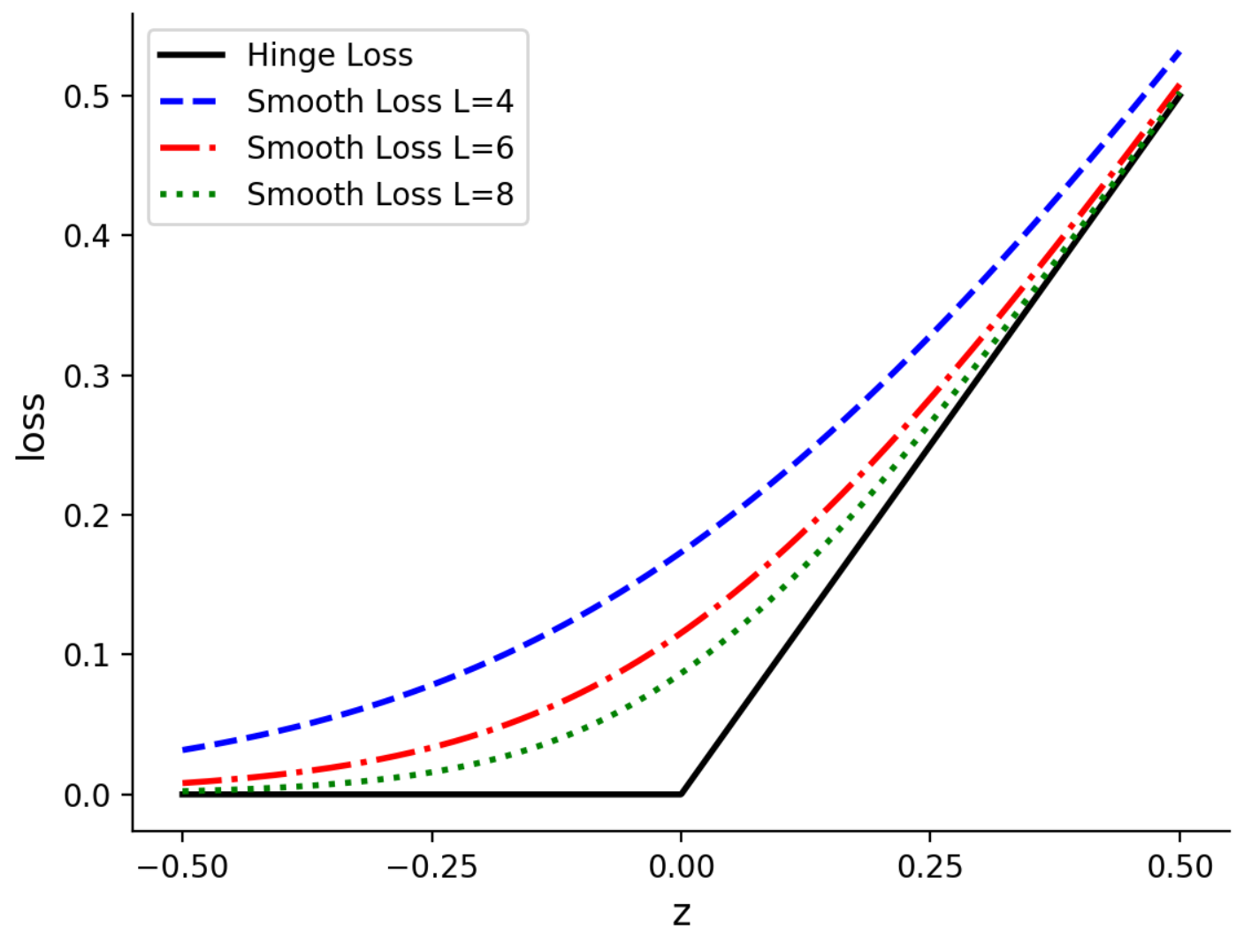}
\mbox{\footnotesize (b) Logistic Loss}
\end{minipage}
\caption{Illustration of the hinge loss and its smooth variants.\label{fig:lossC}}
\end{figure}

Incorporating all of these components, our distributional ranking loss can be defined as
\begin{eqnarray}\label{eq:final}
\min_\theta \mathcal{L}_{\mathrm{DR}}(\theta)=\sum_i^N\ell_{\mathrm{logistic}}(\hat{P}_{i,-}-\hat{P}_{i,+}+\gamma)
\end{eqnarray}
where $\hat{P}_{i,-}$ and $\hat{P}_{i,+}$ are given in Eqn.~\ref{eq:mexample} and $\ell_{\mathrm{logistic}}(\cdot)$ is in Eqn.~\ref{eq:Lloss}. If there is no positive examples in an image, we will let $\hat{P}_{i,+}=1$. Compared with the conventional ranking loss, we rank the expectations of two distributions. It shrinks the number of pairs to $1$ that leads to the efficient optimization.

The gradient of the objective in Eqn.~\ref{eq:final} is easy to compute. The detailed calculation of the gradient can be found in the appendix.

If optimizing the DR loss by the standard SGD with mini-batch as $\theta_{t+1} = \theta_t -\eta \frac{1}{m}\sum_{s=1}^m \nabla \ell_t^s$, we can show that it can converge as in the following theorem. The norm of the gradient is applied to measure the convergence, which is a standard criterion for non-convex optimization~\cite{GhadimiL13a}.  The detailed proof is cast to the appendix.
\begin{thm}\label{thm:1}
Let $\theta_t$ denote the model obtained from the $t$-th iteration with SGD optimizer and the size of mini-batch is $m$. If we assume the objective $\mathcal{L}$ in Eqn.~\ref{eq:final} is $\mu$-smoothness and the variance of the gradient is bounded as $\forall s, \|\nabla \ell_t^s-\nabla\mathcal{L}_t\|_F\leq \delta$, when setting the learning rate as $\eta = \frac{\sqrt{2m\mathcal{L}(\theta_0)}}{\delta\sqrt{\mu T}}$ and $\eta\leq \frac{1}{\mu}$, we have
\[\frac{1}{T}\sum_t \|\nabla \mathcal{L}(\theta_t)\|_F^2 \leq \frac{2\delta\sqrt{2\mu}}{\sqrt{mT \mathcal{L}(\theta_0)}}\]
\end{thm}
\paragraph{Remark 2} Theorem~\ref{thm:1} implies that the learning rate depends on the mini-batch size and the number of iterations as $\eta = \OO(\sqrt{m/T})$ and the convergence rate is $\OO(1/\sqrt{mT})$, where $mT/N$ is the number of training epochs. 

We can obtain a scaling strategy for the learning rate. Let $\eta_0$, $m_0$ and $T_0$ denote a default configuration for training. If we change the mini-batch size as $m' = m_0/\alpha$, where $\alpha$ is a non-negative constant, and keep the same number of epochs for training (i.e., $T' = \alpha T_0$), the convergence rate remains the same. However, the learning rate becomes $\eta' = \OO(\sqrt{m'/T'}) = \eta_0/\alpha$. It shows that to obtain the same performance with a different mini-batch size, we have to rescale the learning rate with a corresponding factor, which is consistent with the observation in \cite{GoyalDGNWKTJH17}. Besides, the learning rate should be no larger than $1/\mu$, which means that the scaling strategy is inapplicable when the mini-batch size is too large.

\subsection{Recover Classification from Ranking}

Detection is to identify foreground objects from background. Therefore, the results from ranking have to be converted to classification. A straightforward way is to set a threshold for all ranking scores. However, the range of ranking scores from different images can vary due to the image-dependent mechanism, and should be calibrated for classification. We investigate the bound for the ranking scores of positive and negative examples as follows.
\begin{thm}\label{thm:rank}
When optimizing the ranking problem as
\[\forall j_+,j_-, \quad p_{j_+}-p_{j_-}\geq \gamma\]
it implies
\[\forall j_+,\quad p_{j_+}> \gamma; \quad \forall j_-,\quad  p_{j_-}\leq 1-\gamma\]
\end{thm}
Therefore, we can recover the standard classification criterion by setting a large margin.
\begin{cor}\label{cor:1}
If setting the margin of ranking as $\gamma = 0.5$, we can recover the classification criterion for ranking scores
\[\forall j_+,\quad p_{j_+}> 0.5;\quad \forall j_-,\quad p_{j_-}\leq 0.5\]
\end{cor}

With these appropriate settings, our final objective for detection can be summarized as
\[\min \sum_i^N \tau \ell_{\mathrm{DR}}^i + \ell_{\mathrm{Reg}}^i\]
where $\ell_{\mathrm{Reg}}$ is the original regression loss in RetinaNet and we keep it unchanged. $\tau$ is the parameter for balancing the weights between classification and regression. We fix it as $\tau=4$ in the experiments.

\section{Experiments}
\label{sec:exp}

\subsection{Implementation Details}
We evaluate the proposed DR loss on COCO data set~\cite{LinMBHPRDZ14}, which contains about $118k$ images for training, $5k$ images for validation, and $40k$ images for test. To focus on the comparison of loss functions, we employ the structure of RetinaNet~\cite{LinGGHD17} as the backbone and only substitute the corresponding focal loss. For a fair comparison, we implement our algorithm in a public codebase~\footnote{https://github.com/facebookresearch/maskrcnn-benchmark}. Besides, we train the model with the same configuration as RetinaNet. Specifically, the model is learned with SGD on $8$ GPUs and the mini-batch size is set as $16$ where each GPU can hold $2$ images at each iteration. Most of experiments are trained with $90k$ iterations that is denoted as ``$1\times$''. The initial learning rate is $0.01$ and is decayed by a factor of $10$ after $60k$ iterations and then $80k$ iterations. For anchor density, we apply the same setting as in \cite{LinGGHD17}, where each location has $3$ scales and $3$ aspect ratios. The standard COCO evaluation criterion is used to compare the performance of different methods.

\subsection{Parameters in DR Loss}
From the definition in Eqn.~\ref{eq:final}, DR loss has three parameters $\lambda_+$, $\lambda_-$ and $L$ to be tuned. $\lambda_+$ and $\lambda_-$ regularize the distribution of scores for positive and negative examples, respectively. $L$ controls the smoothness of the loss function. The margin $\gamma$ is fixed as $0.5$ according to Corollary~\ref{cor:1}. Compared with the focal loss~\cite{LinGGHD17}, DR loss has one more parameter. However, RetinaNet lacks optimizing the relationship between positive and negative distributions, and it has an additional parameter to initialize the output probability of the classifier (i.e., $0.01$) to fit the distribution of background. In contrast, we initialize the probability of the sigmoid function at $0.5$, which is the default threshold for binary classification scenario without any prior knowledge. It verifies that the proposed DR loss can handle the imbalance problem better. Consequently, DR loss roughly has the same number of parameters as that in focal loss.

We will have the ablation study on these parameters to illustrate the influence in the next subsections. Note that RetinaNet applies Feature Pyramid Network (FPN)~\cite{LinDGHHB17} to obtain multiple scale features. To compute DR loss in one image, we collect anchors from multiple pyramid levels and obtain a single distribution for positive and negative anchors, respectively.

\subsection{Effect of Parameters}

We conduct ablation experiments to evaluate the effect of multiple parameters on the \textit{minival} set. All experiments in this subsection are implemented with a single image scale of $800$ for training and test. ResNet-$50$~\cite{HeZRS16} is applied as the backbone for comparison. Only horizontal flipping is adopted as the data augmentation in this subsection. 

\paragraph{Effect of $\lambda_+$ and $\lambda_-$:}
First, we evaluate the effect of $\lambda_+$ and $\lambda_-$ in Eqn.~\ref{eq:mexample}. These parameters constrain the freedom of the derived distributions. As illustrated in Fig.~\ref{fig:ndist}, variances of distributions will have the impact on selecting appropriate weights for regularizers. We investigate the variance of positive and negative anchors, and observe that the standard deviation of positive anchors is about $10$ times larger than that of negative ones. So we roughly set $\lambda_+=1$ and $\lambda_-=0.1$ and fine-tune them as $\lambda_+ = 1/\log(h_+)$ and $\lambda_- = 0.1/\log(h_-)$. It is easy to show that this strategy is equivalent to fixing $\lambda_+$ and $\lambda_-$ as $1$ and $0.1$, and changing the base in the definition of the KL-divergence as $\mathrm{KL}(\q||\oo)= \sum_j q_j\log_{h}\frac{q_j}{o_j}$.

We vary $h_+$ and $h_-$ and summarize the results in Table~\ref{ta:lambda}. First, we observe that the default setting with $\lambda_+=1$ and $\lambda_-=0.1$ can outperform focal loss by $1\%$ ,which demonstrates the effectiveness of the proposed DR loss. Second, the performance of our loss is quite stable in a reasonable range. Finally, the distribution of positive anchors is more sensitive to a small $\lambda_+$, which is consistent with the illustration in Fig.~\ref{fig:ndist}. We keep the best settings in the following experiments.

\begin{table}[!ht]
\centering
\begin{tabular}{ll"lllllll}
$h_+$&$h_-$&AP&AP$_{50}$&AP$_{75}$&AP$_{S}$&AP$_{M}$&AP$_{L}$\\\thickhline
e&e&37.1&56.1& 39.4&19.7&40.9& 50.1\\
\textbf{e}&\textbf{3.5}&\textbf{37.4}&\textbf{56.0}&\textbf{40.0}&\textbf{20.8}&\textbf{41.2}&\textbf{50.5}\\
e&5.5&37.2&55.7&40.0&19.6&41.2&50.4\\
e&20&36.6&54.7&39.4&19.6&40.5&50.3\\
5.5&3.5&36.7&55.2&39.4&19.9&40.4&50.0\\
20&3.5&35.6&54.2&38.1&19.2&39.7&48.3
\end{tabular}
\caption{Comparison of $\lambda_+$ and $\lambda_-$ as in Eqn.~\ref{eq:mexample}. Note that we tune the parameters in the form of $\lambda_+ = 1/\log(h_+)$ and $\lambda_- = 0.1/\log(h_-)$. We adopt $1\times$ iterations and ResNet-50 as the backbone in training. Performance on the \textit{minival} is reported for the ablation study.}\label{ta:lambda}
\end{table}

\paragraph{Effect of Smoothness:}
$L$ controls the smoothness of the loss function in Eqn.~\ref{eq:Lloss}. We compare the model with different $L$'s in Table~\ref{ta:L}. We also include the results for the quadratic loss function in Eqn.~\ref{eq:qloss} with different $\rho$'s for comparison. The original hinge loss is denoted as ``Hinge''. First, all smooth loss functions outperform hinge loss. It confirms that smoothness is important for non-convex optimization. Second, the smooth variants of hinge loss surpass focal loss with a significant margin. It is because that DR loss leverages the information from an image rather than an anchor, which can handle the imbalance issue better. Since quadratic loss and logistic loss have the similar performance, we adopt the logistic loss with $L=6$ in the rest experiments.

\begin{table}[!ht]
\centering
\begin{tabular}{l"lllllll}
Loss&AP&AP$_{50}$&AP$_{75}$&AP$_{S}$&AP$_{M}$&AP$_{L}$\\\thickhline
Focal&36.1&55.0&38.7&19.5&39.5&49.0\\
Hinge&35.8&54.0&38.3&19.3&39.5&47.6\\
$\rho=0.2$&36.9&55.5&39.5&21.1&40.7&49.6\\
$\rho=0.5$&37.2&56.0&39.8&21.1&41.1&50.4\\
$L=4$&37.2&55.9&39.9&20.3&41.0&50.3\\
\textbf{$L=6$}&\textbf{37.4}&\textbf{56.0}&\textbf{40.0}&\textbf{20.8}&\textbf{41.2}&\textbf{50.5}\\
$L=8$&37.1&55.7&39.7&19.5&41.2&50.5\\
$L=10$&36.8&55.4&39.4&20.0&40.7&50.0
\end{tabular}
\caption{Comparison of different loss functions. $\rho$ and $L$ are from Eqn.~\ref{eq:qloss} and Eqn.~\ref{eq:Lloss}, respectively.}\label{ta:L}
\end{table}

\begin{table}[!ht]
\centering
\begin{tabular}{l"lllllll}
Pair &AP&AP$_{50}$&AP$_{75}$&AP$_{S}$&AP$_{M}$&AP$_{L}$\\\thickhline
All&12.9&23.0&12.6&8.8&16.7&15.0\\
NegOnly&37.0&55.5&39.5&19.8&40.7&50.5\\
DR&\textbf{37.4}&\textbf{56.0}&\textbf{40.0}&\textbf{20.8}&\textbf{41.2}&\textbf{50.5}
\end{tabular}
\caption{Comparison of different pairing strategies.}\label{ta:pair}
\end{table}

\paragraph{Effect of Pairing Strategy:}
In DR loss, we propose to rank a single pair consisting of expectations from positive and negative distributions. To evaluate the pairing strategy, we compare it to different strategies for ranking. Specifically, we denote optimizing all pairs in Eqn.~\ref{eq:rank} as ``All'', which is corresponding to the standard ranking problem. We also include a variant of DR loss as ``NegOnly'' that pushes distributions for negative anchors only. The objective of NegOnly on a single image can be written as
\begin{eqnarray*}
\min_\theta \frac{1}{n_+}\sum_{j_+}\ell_{\mathrm{logistic}}(\hat{P}_{-}-p_{j_+}+\gamma)
\end{eqnarray*}
The result of optimizing the worst-case scenario in Eqn.~\ref{eq:bound} is not included since training with that fails to obtain the meaningful result.

The comparison is summarized in Table~\ref{ta:pair}. As illustrated in Section \ref{sec:orank}, the conventional ranking algorithm suffers from the intra-class imbalance in the hardness of negative anchors, which results in the poor performance for detection. By mitigating this issue with the proposed strategy, NegOnly can outperform focal loss. It confirms that handling the imbalance in the negative anchors is important and the proposed strategy can serve the purpose well. Finally, we observe that a tailored distribution for positive anchors can further improve the performance as in DR loss.

\begin{table*}[!ht]
\centering
\begin{tabular}{l|l|lll|lll}
Methods&Backbone&AP&AP$_{50}$&AP$_{75}$&AP$_{S}$&AP$_{M}$&AP$_{L}$\\\thickhline
\textit{two-stage detectors}&&&&&&&\\
Faster R-CNN+++~\cite{HeZRS16}&ResNet-101-C4&34.9&55.7&37.4&15.6&38.7&50.9\\
Faster R-CNN w FPN~\cite{LinDGHHB17}&ResNet-101-FPN&36.2&59.1&39.0&18.2&39.0&48.2\\
Deformable R-FCN~\cite{DaiQXLZHW17}&Aligned-Inception-ResNet&37.5&58.0&40.8&19.4&40.1&52.5\\
Mask R-CNN~\cite{HeGDG17}&ResNet-101-FPN&38.2&60.3&41.7&20.1&41.1&50.2\\\hline
\textit{one-stage detectors}&&&&&&&\\
YOLOv2~\cite{RedmonF17}&DarkNet-19&21.6&44.0&19.2&5.0&22.4&35.5\\
SSD513~\cite{LiuAESRFB16}&ResNet-101-SSD&31.2&50.4&33.3&10.2&34.5&49.8\\
AP-Loss~\cite{ChenLLSWD0HZ19}&ResNet-101-FPN&37.4&58.6&40.5&17.3&40.8&51.9\\
RetinaNet~\cite{LinGGHD17}&ResNet-101-FPN&39.1&59.1&42.3&21.8&42.7&50.2\\
RetinaNet~\cite{LinGGHD17}&ResNeXt-32x8d-101-FPN&40.8&61.1&44.1&24.1&44.2&51.2\\
CornerNet~\cite{LawD18}&Hourglass-104&40.5&56.5&43.1&19.4&42.7&53.9\\
FSAF~\cite{ZhuHS19}&ResNet-101-FPN&40.9&61.5&44.0&24.0&44.2&51.3\\
FCOS~\cite{tian2019fcos}&ResNet-101-FPN&41.5&60.7&45.0&24.4&44.8&51.6\\
FCOS~\cite{tian2019fcos}&ResNeXt-32x8d-101-FPN&42.7&62.2&46.1&26.0&45.6&52.6\\
\hline
Dr. Retina&ResNet-101-FPN&41.7&60.9&44.8&23.5&44.9&53.1\\
Dr. Retina&ResNeXt-32x8d-101-FPN&43.1&62.8&46.4&25.6&46.2&54.0\\
\hline
Dr. Retina (\textit{multi-scale test})&ResNet-101-FPN&43.4&62.1&47.0&26.7&46.1&55.0\\
Dr. Retina (\textit{multi-scale test})&ResNeXt-32x8d-101-FPN&\textbf{44.7}&\textbf{63.8}&\textbf{48.7}&\textbf{28.2}&\textbf{47.4}&\textbf{56.2}
\end{tabular}
\caption{Comparison with the state-of-the-art methods on COCO \textit{test-dev} set.}\label{ta:sota}
\end{table*}

\paragraph{Effect of DR Loss:} 
To illustrate the effectiveness of DR loss, we collect the confidence scores of anchors from all images in \textit{minival} and compare the empirical probability density in Fig.~\ref{fig:illdr}. We include the results from cross entropy loss and focal loss in the comparison.

\begin{figure}[!ht]
\centering
\begin{minipage}{1.6in}
\centering
\includegraphics[width=1.5in]{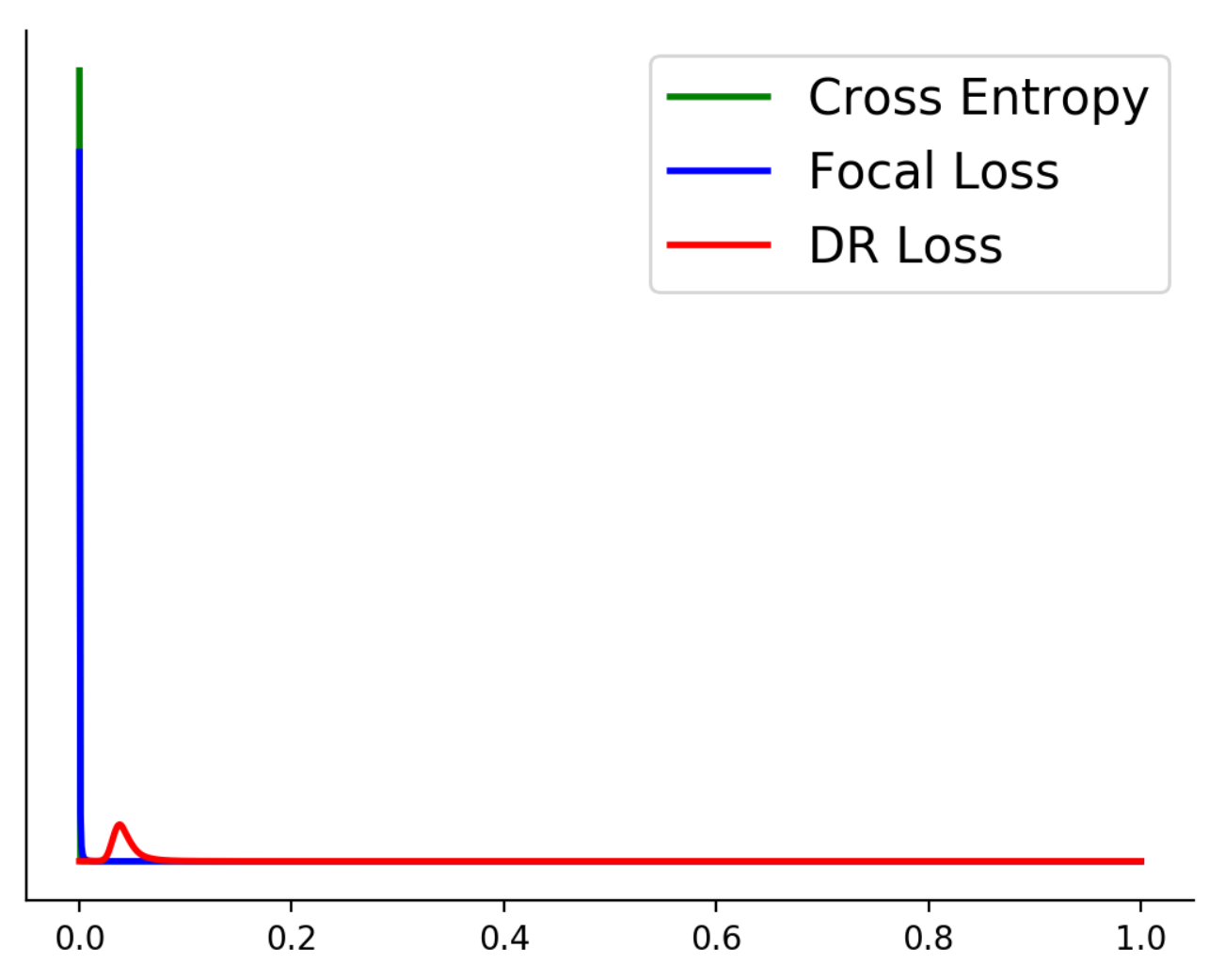}
\mbox{\footnotesize (a) Negative Anchors Distribution  }
\end{minipage}
\begin{minipage}{1.6in}
\centering
\includegraphics[width=1.5in]{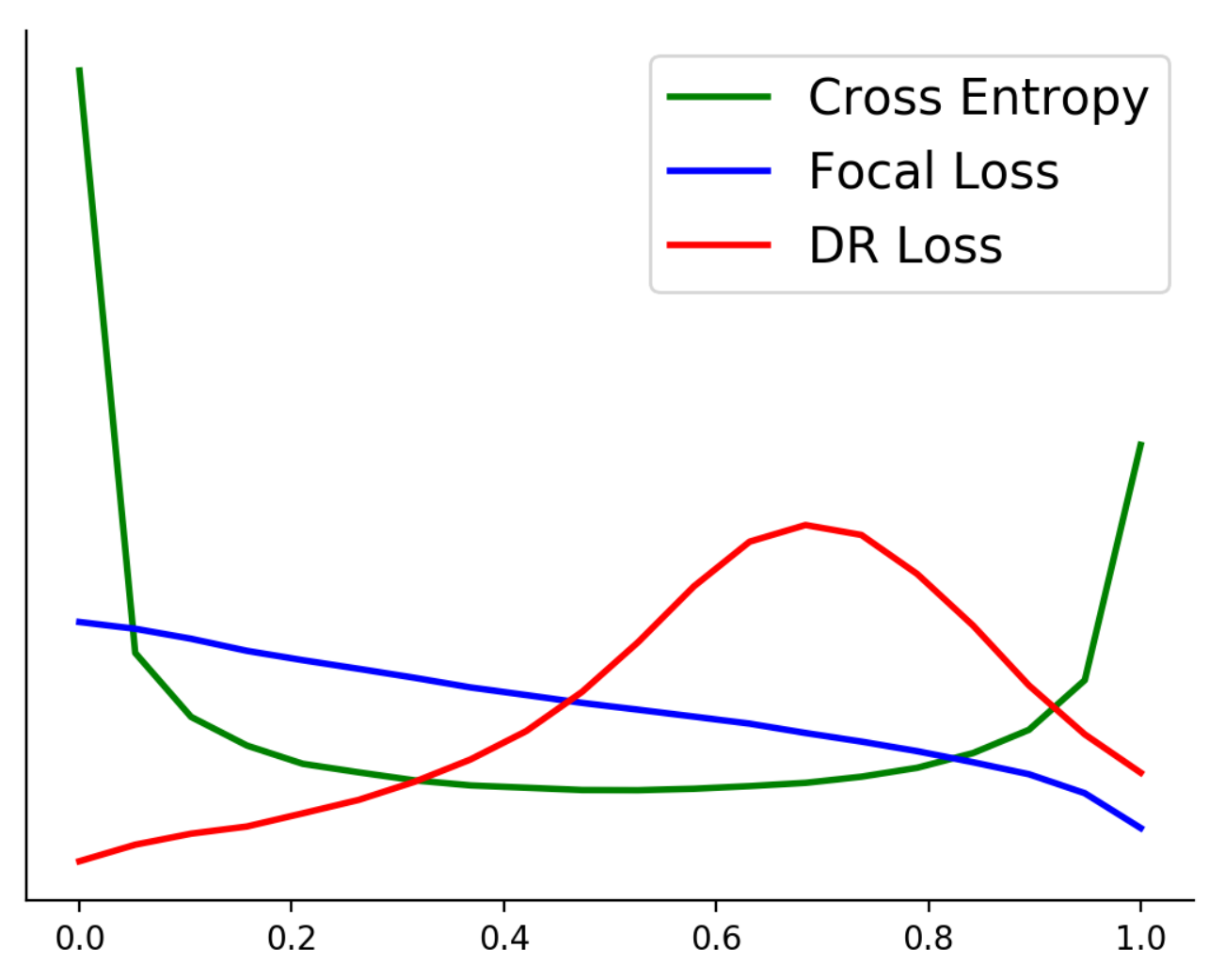}
\mbox{\footnotesize (b) Positive Anchors Distribution  }
\end{minipage}
\caption{Illustration of empirical PDF of distributions that are computed from images in the \textit{minival}.\label{fig:illdr}}
\end{figure}

First, we observe that most of examples have an extremely low confidence after minimizing cross entropy loss. It is because the number of negative examples overwhelms that of positive ones and it will classify most of examples to be negative to obtain a small loss as demonstrated in Eqn.~\ref{eq:imba}. Second, focal loss is better than cross entropy loss by improving the distribution of positive anchors. However, the expectation of the foreground distribution is still close to that of background, and it interprets the fact that focal loss has to initialize the probability of the classifier to be small (i.e., $0.01$). Compared to cross entropy and focal loss, DR loss improves the foreground distribution significantly. By optimizing our ranking loss with a large margin, the expectation of the positive anchors is larger than $0.5$ while that of background is smaller than $0.1$. It confirms that DR loss can address the imbalance between classes well. Besides, the hardness of negative anchors with DR loss is more balanced than that with cross entropy or focal loss. It verifies that with the image-dependent mechanism, DR loss can handle the intra-class imbalance in background examples and focus on the hard negative examples appropriately. More analysis can be found in the appendix.

\subsection{Comparison with State-of-the-Art }
We denote RetinaNet with DR loss as ``Dr. Retina'' and compare it to the state-of-the-art detectors on COCO \textit{test-dev} set. We follow the setting in \cite{LinGGHD17} to increase the number of training iterations to $2\times$, which contains $180k$ iterations, and applies scale jitter in $[640,800]$ as the additional data augmentation for training. Note that we still use a single image scale and a single crop for test as above. Table~\ref{ta:sota} summarizes the comparison for Dr. Retina. With ResNet-101 as the backbone, we can observe that Dr. Retina improves mAP from $39.1\%$ to $41.7\%$. It illustrates that DR loss can explore the imbalance issue in detection more sufficiently than focal loss. Equipped with ResNeXt-32x8d-101~\cite{XieGDTH17} and $1.5\times$ iterations (i.e., $135k$ iterations), the performance of Dr. Retina can achieve $43.1\%$ as a one-stage detector on COCO detection task without bells and whistles. Note that we only replace focal loss with DR loss to obtain the significant gain, which implies that DR loss can be a good substitute of focal loss. Finally, the multi-scale test with scales from $\{400,500,600,700,800,900,1000,1100,1200\}$ can further improve the performance as expected.

\section{Conclusion}
\label{sec:conclusion}
In this work, we introduce the distributional ranking loss to address the imbalance challenge in one-stage object detection. We first convert the original classification problem to a ranking problem, which balances the positive and negative classes. After that, we propose to push the original distributions to the decision boundary and rank the expectations of derived distributions in lieu of original examples to focus on the hard examples, which balances the hardness of background examples. Experiments on COCO verify the effectiveness of the proposed loss function.

\bibliographystyle{ieee_fullname}
\bibliography{dr}

\appendix
\section{Gradient of DR Loss}
We have the DR loss as
\begin{eqnarray*}
\min_\theta \mathcal{L}_{\mathrm{DR}}(\theta)=\sum_i^N\ell_{\mathrm{logistic}}(\hat{P}_{i,-}-\hat{P}_{i,+}+\gamma)
\end{eqnarray*}
where
\begin{eqnarray*}
\ell_{\mathrm{logistic}}(z) = \frac{1}{L}\log(1+\exp(Lz))
\end{eqnarray*}
and
\begin{eqnarray*}
\hat{P}_{i,-} &=& \sum_{j-}^{n_-}\frac{1}{Z_-}\exp(\frac{p_{i,j_-}}{\lambda_-})p_{i,j_-} =  \sum_{j-}^{n_-}q_{i,j_-}p_{i,j_-}\\
\hat{P}_{i,+}&=& \sum_{j_+}^{n_+}\frac{1}{Z_+}\exp(\frac{-p_{i,j_+}}{\lambda_+})p_{i,j_+} = \sum_{j_+}^{n_+}q_{i,j_+}p_{i,j_+}
\end{eqnarray*}

It looks complicated but its gradient is easy to compute. Here we give the detailed gradient.
For $p_{i,j_-}$, we have
\begin{eqnarray*}
&&\frac{\partial \mathcal{L}}{\partial p_{i,j_-}} = \frac{1}{1+\exp(-Lz)}\frac{\partial z}{\partial p_{i,j_-}}\\
&& = \frac{q_{i,j_-}}{1+\exp(-Lz)}(1+\frac{p_{i,j_-}}{\lambda_-} - \frac{1}{\lambda_-}(\sum_{j_-} q_{i,j_-}p_{i,j_-}))
\end{eqnarray*}
where $z = \hat{P}_--\hat{P}_++\gamma$.

For $p_{i,j_+}$, we have
\begin{eqnarray*}
&&\frac{\partial \mathcal{L}}{\partial p_{i,j_+}} = \frac{1}{1+\exp(-Lz)}\frac{\partial z}{\partial p_{i,j_+}}\\
&& = \frac{q_{i,j_+}}{1+\exp(-Lz)}(-1+\frac{p_{i,j_+}}{\lambda_+}-\frac{1}{\lambda_+}(\sum_{j_+} q_{i,j_+}p_{i,j_+}))
\end{eqnarray*}

\begin{table*}[!ht]
\centering
\begin{tabular}{l"llllll|llllll}
&\multicolumn{6}{c|}{Focal Loss}&\multicolumn{6}{c}{DR Loss}\\\hline
Threshold&AP&AP$_{50}$&AP$_{75}$&AP$_{S}$&AP$_{M}$&AP$_{L}$&AP&AP$_{50}$&AP$_{75}$&AP$_{S}$&AP$_{M}$&AP$_{L}$\\\thickhline
0.05&36.1&55.0&38.7&19.5&39.5&49.0 &37.4&56.0&40.0&20.8&41.2&50.5\\
0.1&36.1&54.9&38.7&19.4&39.4&49.0 &37.4&56.0&40.0&20.8&41.2&50.5\\
0.2&35.4&53.4&38.2&18.3&38.7&48.6 &37.4&56.0&40.0&20.8&41.2&50.5\\
0.3&33.9&50.2&37.0&16.2&37.1&47.6 &37.4&56.0&40.0&20.8&41.2&50.5\\
0.4&31.6&45.8&35.0&14.1&34.4&45.2 &37.3&55.9&40.0&20.7&41.2&50.4\\
0.5&28.4&39.7&31.7&10.5&30.5&42.1 &37.2&55.6&39.8&20.1&41.0&50.3
\end{tabular}
\caption{Comparison of different threshold.}\label{ta:thres}
\end{table*}

\begin{table*}[!ht]
\centering
\begin{tabular}{l"llllll|llllll}
&\multicolumn{6}{c|}{Focal Loss~\cite{LinGGHD17}}&\multicolumn{6}{c}{DR Loss}\\\hline
scale&AP&AP$_{50}$&AP$_{75}$&AP$_{S}$&AP$_{M}$&AP$_{L}$&AP&AP$_{50}$&AP$_{75}$&AP$_{S}$&AP$_{M}$&AP$_{L}$\\\thickhline
400&30.5&47.8&32.7&11.2&33.8&46.1&32.4&49.9&34.5&11.7&34.8&48.0\\
500&32.5&50.9&34.8&13.9&35.8&46.7&34.5&52.6&36.6&14.7&36.9&48.9\\
600&34.3&53.2&36.9&16.2&37.4&47.4&36.1&54.6&38.7&17.4&38.5&49.2\\
700&35.1&54.2&37.7&18.0&39.3&46.4&37.1&55.8&39.7&18.9&39.8&49.2\\
800&35.7&55.0&38.5&18.9&38.9&46.3&37.6&56.4&40.3&20.1&40.5&48.9
\end{tabular}
\caption{Comparison of different input scales. We adopt $1\times$ iterations and ResNet-50 as the backbone in training. Results on the \textit{test-dev} are reported.}\label{ta:scale}
\end{table*}

\section{Proof of Theorem~1}
\begin{proof}
First, we give the definition of smoothness
\begin{definition}
A function $F$ is called $\mu$-smoothness w.r.t. a norm $\|\cdot\|$ if there is a constant $\mu$ such that for any $\theta$ and $\theta'$, it holds that
\[F(\theta')\leq F(\theta)+\langle \nabla F(\theta),\theta'-\theta\rangle+\frac{\mu}{2}\|\theta'-\theta\|^2 \]
\end{definition}

We assume that the loss in Eqn.~9 is $\mu$-smoothness, then we have
\begin{eqnarray*}
&&E[\LL(\theta_{t+1})] \leq E[\LL(\theta_{t})+\langle \nabla \LL(\theta_t), \theta_{t+1} - \theta_t\rangle\\
&&\quad\quad\quad\quad\quad\quad+\frac{\mu}{2}\|\theta_{t+1}-\theta_t\|_F^2]\\
&&= E[\LL(\theta_{t})+\langle \nabla \LL(\theta_t),  - \frac{\eta}{m}\sum_{s=1}^m \nabla \ell_t^s \rangle\\
&&\quad\quad\quad\quad\quad\quad+\frac{\mu\eta^2}{2}\|\frac{1}{m}\sum_{s=1}^m \nabla \ell_t^s\|_F^2]
\end{eqnarray*}
According to the definition, we have
\[\forall s, E[\nabla \ell_t^s]=\nabla \LL(\theta_t)\]
If we assume that the variance is bounded as
\[\forall s, \|\nabla \ell_t^s-\nabla \LL_t\|_F\leq \delta\]
then we have
\begin{eqnarray*}
&&E[\LL(\theta_{t+1})] \leq E[\LL(\theta_{t})-\eta \|\nabla \LL_t\|_F^2\\
&&\quad\quad\quad\quad\quad\quad+\frac{\mu\eta^2}{2}\|\frac{1}{m}\sum_{s=1}^m \nabla \ell_t^s-\nabla \LL_t+\nabla \LL_t\|_F^2]\\
&&\leq E[\LL(\theta_{t})-\eta\|\nabla \LL_t\|_F^2+\frac{\mu\eta^2}{2}(\frac{\delta^2}{m}+\|\nabla \LL_t\|_F^2)\nonumber
\end{eqnarray*}

Therefore, we have
\begin{eqnarray*}
(\eta - \frac{\mu\eta^2}{2}) \|\nabla \LL(\theta_t)\|_F^2\leq E[\LL(\theta_t)] - E[\LL(\theta_{t+1})]+\frac{\mu\eta^2\delta^2}{2m}
\end{eqnarray*}
By assuming $\eta\leq \frac{1}{\mu}$ and adding $t$ from $1$ to $T$, we have
\begin{eqnarray*}
\sum_t \|\nabla \LL(\theta_t)\|_F^2 \leq  \frac{2\LL(\theta_0)}{\eta}+\frac{\mu\eta T\delta^2}{m}
\end{eqnarray*}
We finish the proof by letting 
\[\eta = \frac{\sqrt{2m\LL(\theta_0)}}{\delta\sqrt{\mu T}}\]
\end{proof}

\section{Additional Experiments}
\paragraph{Effect of DR Loss:}
We illustrate the empirical PDF of foreground and background from DR loss in Fig.~\ref{fig:illdr2}. Fig.~\ref{fig:illdr2} (a) shows the original density of foreground and background. To make the results more explicit, we decay the density of background by a factor of $10$ and demonstrate the result in Fig.~\ref{fig:illdr2} (b). It is obvious that DR loss can separate the foreground and background with a large margin in the imbalanced scenario. 

\begin{figure}[!ht]
\centering
\begin{minipage}{1.4in}
\centering
\includegraphics[width=1.3in]{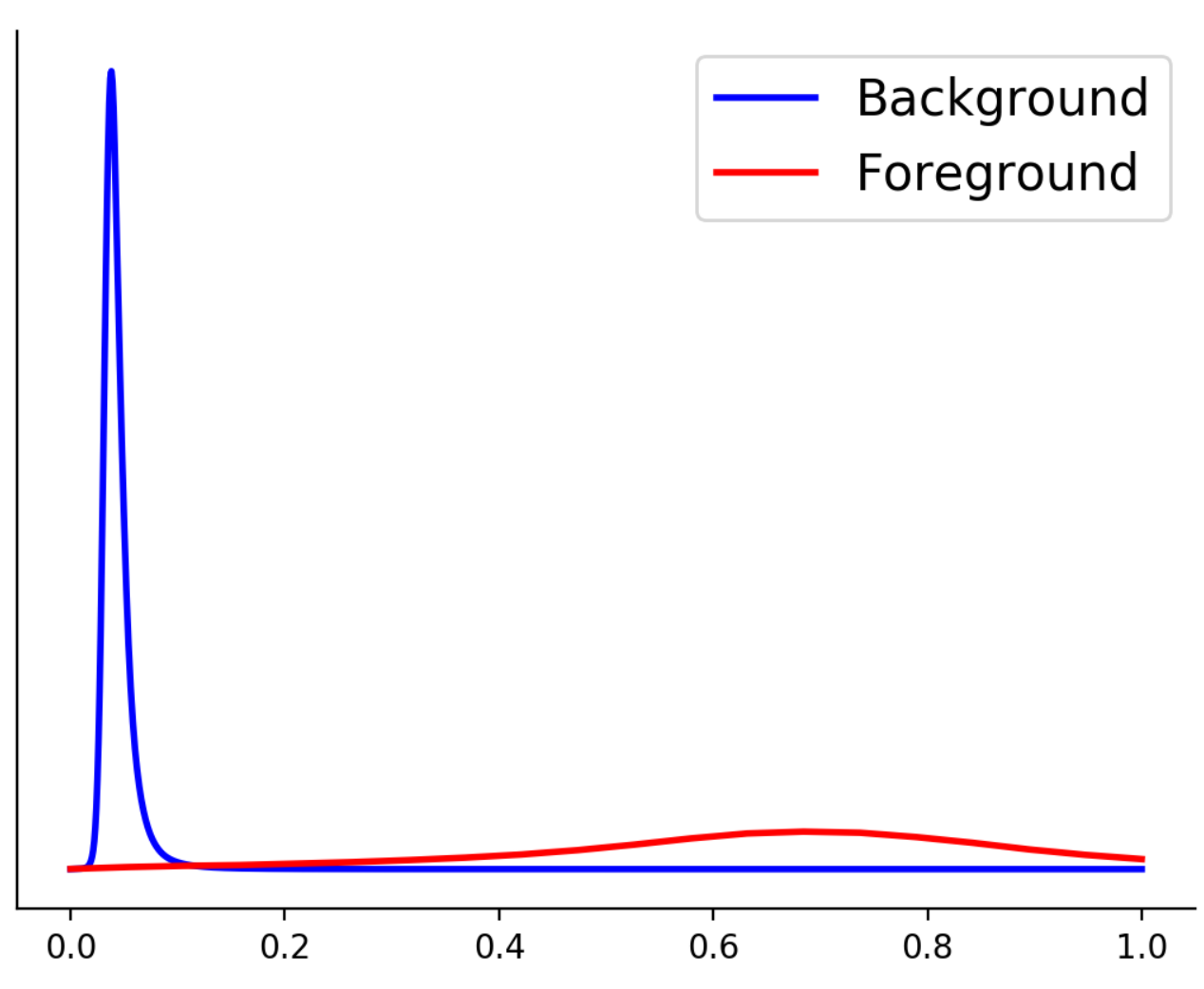}
\mbox{\footnotesize (a) Original Density  }
\end{minipage}
\begin{minipage}{1.4in}
\centering
\includegraphics[width=1.3in]{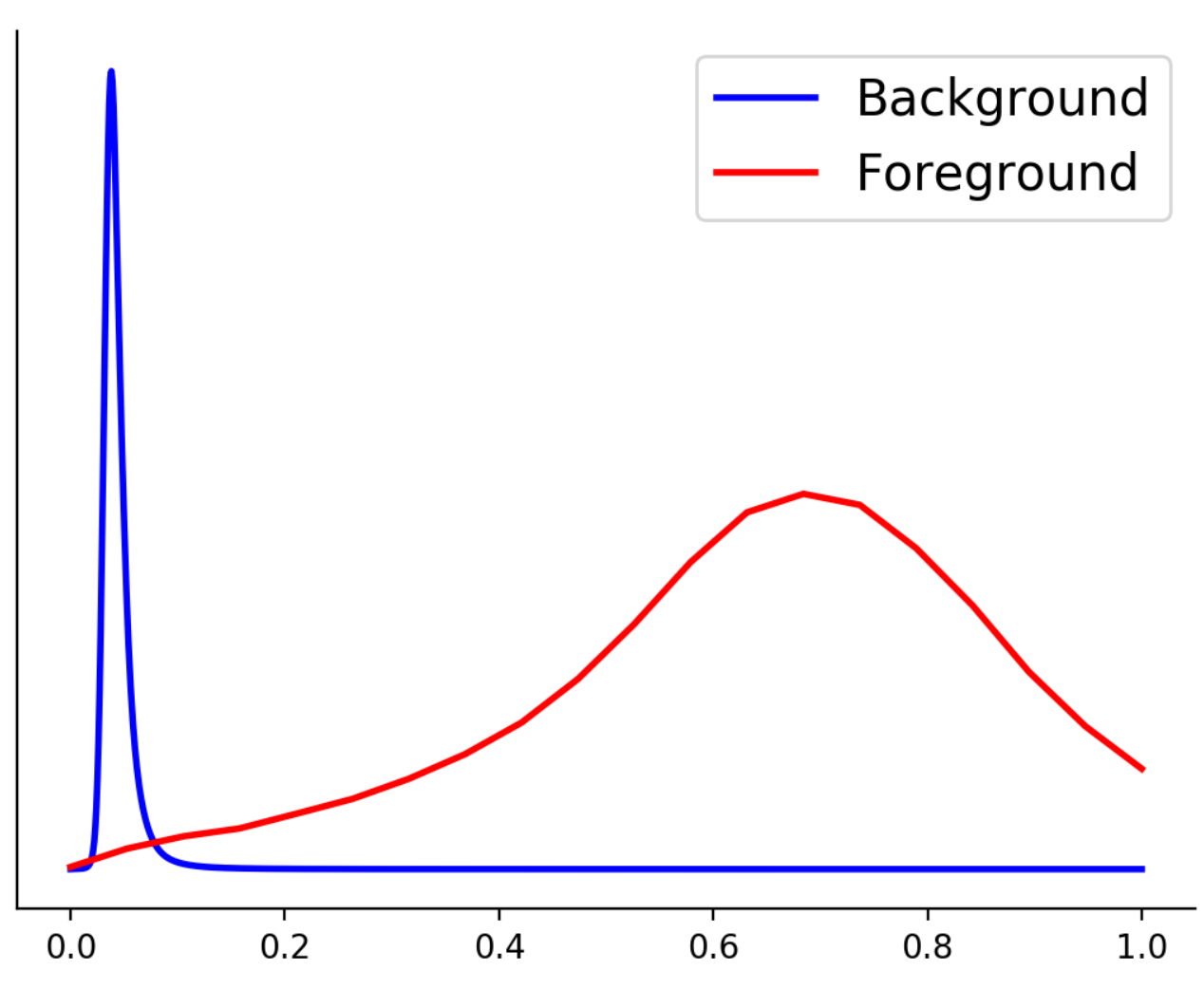}
\mbox{\footnotesize (b) Decayed Density }
\end{minipage}
\caption{Illustration of empirical PDF of distributions from DR loss.\label{fig:illdr2}}
\end{figure}

\paragraph{Effect of Large Margin:}
Before non-maximum suppression (NMS), the candidates with low confidence will be filtered to accelerate detection. Since the distribution of foreground from focal loss is close to that of background as illustrated in Fig.~\ref{fig:illdr}, a small threshold as $0.05$ is adopted to eliminate negative examples. The proposed loss function optimizes the distributions with a large margin and can be robust to the selection of the threshold. Table~\ref{ta:thres} demonstrates the performance with different thresholds. It is obvious that the performance of DR loss keeps almost the same while that of focal loss degrades significantly when increasing the threshold.

\paragraph{Effect of Image Scale:}
We tune the parameters of DR loss with a single input scale of $800$ but the parameters are robust to different input scales. We follow the settings in the ablation study and Table~\ref{ta:scale} compares the performance on \textit{test-dev} with scales varied in $\{400,500,600,700,800\}$. Note that the maximal size of images is also changed with a corresponding factor. We report the results of focal loss from \cite{LinGGHD17}. Evidently, DR loss can consistently improve the performance over focal loss by about $2\%$. It demonstrates that the proposed loss function is not sensitive to the scale of input images.

\paragraph{Comparison on PASCAL:}
Finally, we evaluate the proposed DR loss on a different data set: PASCAL VOC2007~\cite{pascal-voc-2007}, which contains $9,963$ images and $20$ classes. We adopt the same configurations for RetinaNet as in the ablation study and the same parameters as those on COCO for DR loss and focal loss. We change the initial learning rate to $0.008$ and it is decayed at $6,250$ iterations, where the total number of iterations is $8,750$ as suggested by the codebase. Other training settings are the same as the pipeline for COCO. The detector is trained with the training and validation sets, and Table~\ref{ta:pascal} shows the comparison on the test data. We can observe that with the same parameters on a different task, our method can outperform focal loss with a significant margin. It demonstrates that the proposed loss function can be applicable for different tasks.

\begin{table}[ht]
\centering
\begin{tabular}{l"llll}
Loss&AP&AP$_{50}$&AP$_{75}$\\\thickhline
Focal&39.5&67.2&40.8\\
DR&41.2&68.6&42.6
\end{tabular}
\caption{Comparison on VOC2007. Results on the \textit{test} are reported.}\label{ta:pascal}
\end{table}

\end{document}